\documentclass[letterpaper,10pt,conference]{ieeeconf}
\usepackage{cite}
\usepackage{graphicx}
\usepackage{subcaption}
\usepackage[cmex10]{amsmath}
\usepackage{algorithm}
\usepackage{algorithmic}
\usepackage{xcolor}
\usepackage{array}
\usepackage{url}
\usepackage{bbm}
\usepackage{todonotes}

\usepackage{epsfig}
\usepackage{amsfonts,amssymb,multirow,bigstrut,booktabs,ctable,latexsym}
\usepackage{mathtools}
\usepackage[mathscr]{eucal}
\usepackage{verbatim}

\usepackage{amsthm}
\usepackage{algorithm}
\usepackage[normalem]{ulem}
\newtheorem{definition}{Definition}
\newtheorem{lemma}{Lemma}
\newtheorem{theorem}{Theorem}

\newtheorem{proposition}{Proposition}
\newtheorem{remark}{Remark}
\newtheorem{assumption}{Assumption}

\IEEEoverridecommandlockouts

\title{Privacy-Preserving Reinforcement Learning Beyond Expectation}
	
	\author{Arezoo Rajabi$^{1}$, Bhaskar Ramasubramanian$^{2}$, Abdullah Al Maruf$^{1}$, Radha Poovendran$^{1}$%
		\thanks{$^{1}$Network Security Lab, Department of Electrical and Computer Engineering, 
			University of Washington, Seattle, WA 98195, USA. \newline
			{\tt\small \{rajabia, maruf3e, rp3\}@uw.edu}}
\thanks{$^2$Electrical and Computer Engineering, Western Washington University, Bellingham, WA 98225, USA.
			{\tt\{ramasub\}@wwu.edu}}
	}
\begin{document}	
\maketitle

\begin{abstract}
Cyber and cyber-physical systems equipped with machine learning algorithms such as autonomous cars share environments with humans. 
In such a setting, it is important to align system (or agent) behaviors with the preferences of one or more human users. 
We consider the case when an agent has to learn behaviors in an unknown environment. 
Our goal is to capture two defining characteristics of humans: i) a tendency to assess and quantify risk, and ii) a desire to keep decision making hidden from external parties. 
We incorporate cumulative prospect theory (CPT) into the objective of a reinforcement learning (RL) problem for the former. 
For the latter, we use differential privacy. 
We design an algorithm to enable an RL agent to learn policies to maximize a CPT-based objective in a privacy-preserving manner and establish guarantees on the privacy of value functions learned by the algorithm when rewards are sufficiently close. 
This is accomplished through adding a calibrated noise using a Gaussian process mechanism at each step. 
Through empirical evaluations, we highlight a privacy-utility tradeoff and demonstrate that the RL agent is able to learn behaviors that are aligned with that of a human user in the same environment in a privacy-preserving manner. 
\end{abstract}

\section{Introduction}\label{Sec:Intro}

Complex cyber and cyber-physical systems (CPS), including autonomous cars and drones, rely on the seamless integration of computation and physical components. 
A CPS might depend on machine learning algorithms for decision making due to the large amounts of data generated during its operation and limited access to models of its environment. 
Reinforcement learning (RL) \cite{sutton2018reinforcement} and optimal control \cite{bertsekas2017dynamic} are two paradigms that have been commonly leveraged to choose actions to maximize an expected reward over the horizon of system operation 
when dynamic system behaviors are represented as Markov decision processes (MDP) \cite{puterman2014markov}.

A \emph{risk-neutral} approach to decision making involves learning strategies (sequence of actions) to maximize an expected reward, where the reward signal is provided by the environment. 
Learning strategies to satisfy risk-neutral objectives have been successfully implemented in multiple domains, including robotics, autonomous vehicles, games, and mobile networks \cite{hafner2011reinforcement, mnih2015human, silver2016mastering, zhang2019deep, sadigh2016planning, yan2018data, you2019advanced}. 
Although these methods to learn strategies are tractable and efficient, rational and risk-neutral decision making using the expected utility is often not sufficient to model decision making in real-world CPS. 

Decision making in realistic settings is often risk-sensitive. 
It is becoming increasingly common for learning-based CPS to share an environment with human users \cite{seshia2015formal, nikolakis2019cyber, xiao2020fresh}. 
In such a situation, it will become important for the system to be aware of risk-sensitive and irrational behaviors of users. 
Due to various cognitive and emotional biases, human users can exhibit both risk-seeking and risk-averse behaviors. 
In these cases, expected utility-based frameworks are not adequate to describe human decision making, since humans might have a different perception of both, the utility and the probabilistic outcome as a consequence of their decisions \cite{kahneman1979prospect}. 

To effectively capture preferences of humans for certain outcomes over certain others, we use insights from empirical models of human behavior from the social sciences. 
These models have shown that humans derive utility relative to a reference point \cite{schmidt2003reference}. 
There is also a tendency to be more sensitive to losses than gains, and rather than using actual probabilities to assess outcomes, very small and very large probabilities are distorted \cite{barberis2013thirty}. 
As an illustration, human drivers on the road i) are more sensitive to changes in speed, than its absolute value; ii) are more averse to being passed (loss) than passing another car (gain); iii) overestimate small probabilities of engine failures and underestimate large probabilities of running out of gas. 
Cumulative prospect theory (CPT), introduced in \cite{tversky1992advances}, is a framework that incorporates the above properties. 
CPT uses a non-linear utility function to transform outcomes and a non-linear weighting function to distort probabilities in the cumulative distribution function. 
Utility and weighting functions corresponding to gains and losses can be different to model the possibility that these are often interpreted in different ways by a human. 

Our recent work in \cite{ramasubramanian2021reinforcement} developed RL algorithms for CPT-based decision making when a model of the system was not known. 
We established theoretical guarantees on convergence and demonstrated that behaviors of agents using CPT-based policies closely mimicked those of a human user in the same environment. 
In this paper, we focus on another defining characteristic of human behavior: \emph{a desire to keep decision-making and its outcomes hidden from external, and possibly adversarial parties}. 
We use the foundations and tools of differential privacy \cite{dwork2014algorithmic}, a security property that makes it difficult for an adversary to discern information about a system by providing probabilistic guarantees on the indistinguishability of its observations. 
This makes it unlikely that an adversary will learn anything of specific interest or meaningful about sensitive data \cite{liao2019prospect}. 

To assimilate this characteristic within an RL framework, it is important to identify the sensitive information that might be gleaned through knowledge of individual components of the decision-making procedure. 
Rewards received by an agent are an important descriptor of the task that needs to be completed; transition probabilities from an MDP representation reveal information about the consequence of taking a particular action; a trajectory of states visited by the agent can disclose tracking information \cite{wang2019privacy}. 
Previous research has established that the rewards are the most vulnerable component \cite{ng2000algorithms, abbeel2004apprenticeship}. 
Consequently, we propose a method to incorporate differential privacy into risk-sensitive reinforcement learning, to ensure that rewards which are `sufficiently close' to each other will be indistinguishable. 

To the best of our knowledge, the incorporation of a privacy-preserving mechanism into prospect-theoretic reinforcement learning has not been previously investigated. Such a framework will ensure that an autonomous agent can learn behaviors consistent with two defining characteristics of humans: i) a propensity to assess and quantify risk, and ii) a desire to keep decision making hidden from external parties. 
We make the following contributions: 
\begin{itemize}
\item We design an algorithm, \emph{PP-CPT-RL}, to enable an RL agent to learn policies to maximize a CPT-based objective in a privacy-preserving manner. Our algorithm adds noise to the CPT-value of a state-action pair at each step through a Gaussian process mechanism. 

\item We establish guarantees on the privacy of value functions learned by PP-CPT-RL when rewards are sufficiently close to each other using differential privacy. 

\item We evaluate the PP-CPT-RL algorithm in a continuous-state environment where an agent has to reach a target while avoiding obstacles. 
Our experiments highlight a privacy-utility tradeoff, and we demonstrate that PP-CPT-RL allows agents to learn optimal policies while maintaining indistinguishability of rewards they receive. Moreover, agent behaviors are aligned with those of a human who might be placed in the same environment. 
\end{itemize}

The remainder of this paper is organized as follows: 
Section \ref{Sec:RelWork} provides an overview of related literature on risk-sensitive RL and differential privacy. 
Section \ref{Sec:Preliminaries} establishes notation and describes necessary background material. 
We summarize our CPT-based RL framework from \cite{ramasubramanian2021reinforcement} in Section \ref{Sec:CPT-RL}. 
Section \ref{Sec:PP-CPT-RL} details the design of the PP-CPT-RL algorithm, and establishes guarantees on its privacy-preserving nature. 
We report results of our empirical evaluations in Section \ref{Sec:Expt}, and Section \ref{Sec:Conclusion} presents our conclusions. 

\section{Related Work}\label{Sec:RelWork}

This section summarizes related work in risk-sensitive and differentially private reinforcement learning. 

Incorporating risk into decision-making has been typically carried out by replacing the average utility with the average over a function of the utility \cite{shen2013risk, shen2014risk}. 
Examples include a mean-variance tradeoff \cite{markowitz1952portfolio, tamar2012policy, mannor2013algorithmic}, exponential 
utility \cite{howard1972risk, whittle1990risk, borkar2002q}, and conditional value at risk (CVaR) \cite{rockafellar2002conditional}. 
CVaR is the average cost, given that the cost takes sufficiently large values, and has strong theoretical justification for its use \cite{ahmadi2021constrained, chapman2021risk, lindemann2020control}. 
An axiomatic characterization of risk considerations for a robot was presented in \cite{majumdar2020should}. 
Risk-sensitivity has also been represented as a constraint to be satisfied while maximizing an average utility \cite{chow2017risk, prashanth2018risk}. 
Cumulative prospect theory (CPT) \cite{tversky1992advances} models behaviors of a decision-maker that is risk-averse with gains and risk-seeking with losses, and distorts extremely high and low probability events. 
Optimization of a CPT-based cost for MDPs was studied in \cite{jie2018stochastic, lin2018probabilistically}. 
Our previous work \cite{ramasubramanian2021reinforcement} optimized a CPT-based cost in an RL framework, and designed CPT-RL algorithms; a similar paradigm was concurrently proposed in \cite{borkar2021prospect}.  
While the above works provide promising solutions for risk-sensitive decision makers, incorporating CPT-based objectives into an RL framework has been relatively less studied. 
Further, the consideration of a desire of an RL agent to maintain privacy of its decision making has not been explored using CPT. 

Differential privacy has been used to reason about indistinguishability of trajectories of a dynamical system in \cite{cortes2016differential, yazdani2018differentially}. 
An overview of results that use differential privacy in control is presented in the survey \cite{han2018privacy}.
A characterization of differential privacy for discrete-state Markov chains was presented in \cite{chistikov2019asymmetric}, and this was extended to trajectories of discrete MDPs in our previous work \cite{ramasubramanian2020privacy}. 
Algorithmic guarantees on the differential privacy of policies in RL using Monte-Carlo techniques were provided in \cite{balle2016differentially}. 
An algorithm to synthesize privacy-preserving exploration policies that simultaneously achieved strong regret bounds for episodic RL in discrete environments was proposed in \cite{vietri2020private}. 
A lower bound for regret minimization in finite horizon MDPs with strong privacy-preserving guarantees was presented in \cite{garcelon2021local}. 
The authors of \cite{gohari2020privacy} developed a policy synthesis algorithm that protected the privacy of transition probabilities of MDPs. 
Guarantees on the privacy of value functions of a deep Q-learning algorithm in environments with continuous states using differential privacy were established in \cite{wang2019privacy}. 
Although the above works integrate differential privacy into an MDP or an RL framework, these examine the case where an expected reward needs to be maximized. 
In comparison, we establish privacy guarantees on value functions of a CPT-based deep RL algorithm in environments with continuous states.

\section{Preliminaries}\label{Sec:Preliminaries}

This section introduces background on reinforcement learning, cumulative prospect theory, and differential privacy.

\subsection{MDPs and RL}

Let $(\Omega, \mathcal{F}, \mathcal{P})$ denote a probability space, where $\Omega$ is a sample space, $\mathcal{F}$ is a $\sigma-$algebra of subsets of $\Omega$, and $\mathcal{P}$ is a probability measure on $\mathcal{F}$. 
A random variable (r.v.) is a map $Y: \Omega \rightarrow \mathbb{R}$. 
We assume that the environment of the RL agent is described by a Markov decision process (MDP) \cite{puterman2014markov}. 

\begin{definition}
An MDP is a tuple $\mathcal{M}:= (S, A, \rho_0, \mathbb{P}, r, \gamma)$, where $S$ is a finite set of states, $A$ is a finite set of actions, and $\rho_0$ is a probability distribution over the initial states. 
$\mathbb{P}(s'|s,a)$ is the probability of transiting to state $s'$ when action $a$ is taken in state $s$. 
$r: S \times A \rightarrow \mathbb{R}$ is the reward obtained by the agent when it takes action $a$ in state $s$. 
$\gamma \in (0,1]$ is a discounting factor.
\end{definition}

An RL agent typically does not have knowledge of the transition function $\mathbb{P}$. 
Instead, it obtains a (finite) reward $r$ for each action that it takes. 
Through repeated interactions with the environment, the agent seeks to learn a policy $\pi$ in order to maximize an objective $\mathbb{E}_\pi[\sum_t \gamma^t r(s_t,a_t)]$ \cite{sutton2018reinforcement}. 
A \emph{policy} is a probability distribution over the set of actions at a given state, and is denoted $\pi(\cdot|s)$. 
In realistic scenarios, the expected reward might not be an adequate representation of decision-making. 
This will necessitate the incorporation of risk-sensitivity into the RL framework.

\subsection{Risk Measures and Cumulative Prospect Theory}

For a set $\mathcal{Y}$ of random variables on $\Omega$, a \emph{risk measure} or \emph{risk metric} is a map $\rho: \mathcal{Y} \rightarrow \mathbb{R}$ \cite{majumdar2020should}. 

\begin{definition}
A risk metric is coherent if it satisfies the following properties for all $Y, Y_1, Y_2 \in \mathcal{Y}, d \in \mathbb{R}, m \in \mathbb{R}_{\geq 0}$: 
\begin{enumerate}
\item Monotonicity: $Y_1(\omega) \leq Y_2(\omega)$ for all $\omega \in \Omega$ $\Rightarrow$ $\rho(Y_1) \leq \rho(Y_2)$; 
\item Translation invariance: $\rho(Y+d) = \rho(Y) + d$; 
\item Positive homogeneity: $\rho(mY) = m \cdot \rho(Y)$; 
\item Subadditivity: $\rho(Y_1+Y_2) \leq \rho(Y_1) + \rho(Y_2)$. 
\end{enumerate}
\end{definition}

The last two properties together ensure that a coherent risk metric will also be convex. 
%
%
%
The risk metric that we adopt in this paper is informed from cumulative prospect theory \cite{tversky1992advances}, and is not coherent. 
%
%
Human players or operators have been known to demonstrate a preference to play safe with gains and take risks with losses. 
Further, they tend to \emph{deflate} high probability events, and \emph{inflate} low probability events. 
%
%
Cumulative prospect theory (CPT) is a risk measure that has been empirically shown to capture human attitude to risk \cite{tversky1992advances, jie2018stochastic}. 
This risk metric uses two \emph{utility functions} $u^+$ and $u^-$, corresponding to gains and losses, and \emph{weight functions} $w^+$ and $w^-$ that reflect the fact that value seen by a human subject is nonlinear in the underlying probabilities \cite{barberis2013thirty}. 

\begin{definition}\label{CPTValueDefn}
The CPT-value of a continuous r.v. $Y$ is: 
{\small
\begin{align}
\rho_{cpt}(Y)&:= \int_0^\infty w^+(\mathbb{P}(u^+(Y) > z))dz \nonumber \\&\qquad \qquad- \int_0^\infty w^-(\mathbb{P}(u^-(Y) > z))dz,  \label{CPTValue}
\end{align}
}%
where utility functions $u^+, u^- : \mathbb{R} \rightarrow \mathbb{R}_{\geq 0}$ are continuous, have bounded first moment such that $u^+(x) = 0$ for all $x \leq 0$, and monotonically non-decreasing otherwise, and $u^-(x) = 0$ for all $x \geq 0$, and monotonically non-increasing otherwise. 
The probability weighting functions $w^+, w^-: [0,1] \rightarrow [0,1]$ are Lipschitz continuous and non-decreasing, and satisfy $w^+(0) = w^-(0) = 0$ and $w^+(1) = w^-(1) = 1$. 
\end{definition}

When $Y$ is a discrete r.v. with finite support, let $p_i$ denote the probability of incurring a gain or loss $y_i$, where $y_1 \leq \dots \leq y_l \leq 0 \leq y_{l+1} \leq \dots y_K$, for $i = 1,2,\dots,K$. 
Define $F_k:= \sum_{i=1}^k p_i$ for $k \leq l$ and $F_k:= \sum_{i=k}^K p_i$ for $k > l$. 

\begin{definition}
The CPT-value of a discrete r.v. $Y$ is:
{\small
\begin{align}
&\rho_{cpt}(Y)\label{CPTValueDiscrete}\\&:= \bigg(\sum_{i=l+1}^{K-1} u^+(y_i) \big(w^+(F_i) - w^+(F_{i+1}) \big)+u^+(y_K)w^+(p_K) \bigg) \nonumber \\
&\quad- \bigg(u^-(y_1)w^-(p_1)+\sum_{i=2}^{l} u^-(y_i) \big(w^-(F_i) - w^-(F_{i-1}) \big) \bigg) \nonumber
\end{align}
}%
\end{definition}

The function $u^+$ is typically concave on gains, while $-u^-$ is typically convex on losses \cite{tversky1992advances}. 
The distortion of extremely low and extremely high probability events by humans can be represented by a weight function that takes an \emph{inverted S-shape}- i.e., it is concave for small probabilities, and convex for large probabilities (e.g., $w(\kappa)= \exp(-(-\ln \kappa)^\eta), 0 < \eta < 1$) \cite{tversky1992advances, prelec1998probability}. 
%
%
The CPT-value generalizes other risk metrics 
for appropriate choices of weighting functions. 
For example, when $w^+, w^-$ are identity functions, and $u^+(x) = x, x \geq 0$, $u^-(x) = -x, x \leq 0$, we obtain $\rho_{cpt}(Y) = \mathbb{E}[Y]$. 


\subsection{Gaussian Processes} \label{GausProc}

A Gaussian vector-valued random variable $Y$ is denoted $\mathcal{N}(\mu, \Sigma)$, where $\mu$ is the mean vector, and $\Sigma$ is a symmetric, positive definite covariance matrix. 
Consider a partition of the Gaussian random vector into two sets such that $Y = \begin{bmatrix}Y_1&Y_2 \end{bmatrix}^T$ with corresponding partitions of their means and covariances as $\mu = \begin{bmatrix}\mu_1&\mu_2 \end{bmatrix}^T$ and $\Sigma = \begin{bmatrix}\Sigma_{11}&\Sigma_{12}\\ \Sigma_{12}^T&\Sigma_{22}\end{bmatrix}$, then the following properties hold: 
\begin{enumerate}
\item $Y_1 \sim \mathcal{N}(\mu_1, \Sigma_{11})$ and $Y_2 \sim \mathcal{N}(\mu_2, \Sigma_{22})$
\item $Y_1|Y_2 \sim \mathcal{N}(\mu_1+\Sigma_{12}\Sigma_{22}^{-1}(Y_2-\mu_2), \Sigma_{11}-\Sigma_{12}\Sigma_{22}\Sigma_{12}^T)$
\item $Y_2|Y_1 \sim \mathcal{N}(\mu_2+\Sigma_{12}^T\Sigma_{11}^{-1}(Y_1-\mu_1), \Sigma_{22}-\Sigma_{12}^T\Sigma_{11}\Sigma_{12})$
\end{enumerate}

Gaussian processes (GPs) generalize the concept of a Gaussian distribution over discrete random variables to the idea of a Gaussian distribution over continuous functions \cite{williams2006gaussian}. 
GPs are particularly useful in uncertainty quantification due to their ability to simultaneously track the evolution of the mean and the covariance of a distribution. 

\begin{definition}\label{GaussProcDefn}
A Gaussian process (GP) is a collection of r.v., any finite number of which have a joint Gaussian distribution. 
A GP $f(x)$ is specified by its mean $m(x)$ and covariance function $K(x,x')$, and we write $f \sim \mathcal{G}(m, K)$. 
\end{definition}

For 
$f \sim \mathcal{G}(g, \sigma^2K)$, 
let $f_{n0}:=\{f(x_0), f(x_2),$ $\dots,f(x_{2n})\}$ and $f_{n1}:=\{f(x_1),f(x_3),\dots,f(x_{2n-1})\}$ where $x_i = i/2n$, $i=0,\dots, 2n$. Let $\beta_n:= \beta/2n$. 
Then, from Definition \ref{GaussProcDefn}, $f_{n1}|f_{n0} \sim \mathcal{N}(g_{n1}+K_{10}K_{00}^{-1}(f_{n0}-g_{n0}), \sigma^2(K_{11}-K_{10}K_{00}^{-1}K_{10}^T))$. 

\subsection{Differential Privacy}

Differential privacy is a property that ensures that private data of an agent is protected, while allowing for statistical inferences from aggregates of the data \cite{dwork2014algorithmic}. 
This makes it unlikely that an adversary will learn anything meaningful about sensitive data. 
Attractive features of differential privacy include compositionality, resilience to post-processing, and robustness to side information. 
The notion of differential privacy is mathematically defined using a notion of \emph{neighboring data points} to characterize data points that are sufficiently close to each other according to some metric/ norm. 
Let $d, d' \in \mathcal{D}$ be neighboring inputs.

\begin{definition} \label{DiffPrivDefnCts} 
A randomized mechanism $\mathbb{M}:\mathcal{D} \rightarrow \mathcal{U}$ satisfies $(\epsilon, \delta)-$differential privacy if for any two neighboring inputs $d, d'$ and for any subset of outputs $\mathcal{Z} \subseteq \mathcal{U}$, $\mathbb{P}(\mathbb{M}(d) \in \mathcal{Z}) \leq exp(\epsilon)\mathbb{P}(\mathbb{M}(d') \in \mathcal{Z}) + \delta$. 

The sensitivity of a mechanism $\mathbb{M}$ is defined as $\Delta_{\mathbb{M}}:= \sup\{d, d' \in \mathcal{D}: ||\mathbb{M}(d) - \mathbb{M}(d')||\}$ for some norm on $\mathcal{U}$. 
\end{definition}

An example of $\mathbb{M}$ is the Gaussian mechanism. 
When $\mathcal{U} = \mathbb{R}^n$, $\mathcal{N}(0, \sigma^2 I)$ is added to the output $\mathbb{M}(d)$, and the norm on $\mathcal{U}$ is the $\ell^2-$norm. 
When $\mathcal{U}$ is a reproducing kernel Hilbert space (RKHS), the norm is the RKHS norm \cite{hall2013differential}, and a Gaussian process noise $\mathcal{G}(0, \sigma^2K)$ is added to $\mathbb{M}(d)$. 
%

\section{CPT-based Reinforcement Learning}\label{Sec:CPT-RL}

This section introduces a reinforcement learning paradigm that maximizes a CPT-based reward. 
We direct the reader to \cite{ramasubramanian2021reinforcement} for proofs of the results. 
Note that while \cite{ramasubramanian2021reinforcement} minimized the sum of CPT-based costs, in this paper, we maximize the sum of CPT-based rewards. 

In order to assess the quality of taking an action $a$ at a state $s$, we introduce the notion of the \emph{CPT-value of state-action pair at time $t$ and following policy $\pi$} subsequently. 
We denote this by $Q^\pi_{cpt}(s,a)$ and will refer to it as \emph{CPT-Q}. 
\emph{CPT-Q} is defined in the following manner: 
\begin{align}
&Q^\pi_{cpt}(s_t,a_t):= \rho_{cpt}(r(s_t, a_t)\label{CPT-Q-Iterative}\\ &+ \gamma \sum_{s_{t+1}} \mathbb{P}(s_{t+1}|s_t,a_t) \sum_{a_{t+1}} \pi(a_{t+1}|s_{t+1})Q^\pi_{cpt}(s_{t+1},a_{t+1})). \nonumber
\end{align}

$Q^\pi_{cpt}(s,a)$ will be bounded when $|r(s,a)| < \infty$ and $\gamma \in (0,1)$. 
In reinforcement learning, transition probabilities and rewards are typically not known apriori. 
In the absence of a model, the agent will have to estimate $Q^\pi_{cpt}(s,a)$ and learn `good' policies by exploring its environment. 
Since $Q^\pi_{cpt}(s,a)$ is evaluated for each action in a state, this quantity can be estimated without knowledge of the transition probabilities. 
This is in contrast to \cite{lin2018probabilistically}, where a model of the system was assumed to be available, and costs were known.

The \emph{CPT-value of a state $s$ when following policy $\pi$} is defined as $V^\pi_{cpt}(s_t) := \sum_{a_t} \pi(a_t|s_t)Q^\pi_{cpt}(s_t,a_t)$. 
We will refer to $V^\pi_{cpt}(s)$ as \emph{CPT-V}. 
We observe that \emph{CPT-V} satisfies: 
\begin{align}
V^\pi_{cpt}(s_t)&= \rho_{cpt}(r(s_t, a^\pi_t)\label{CPT-V-Iterative}\\&\qquad \qquad+\gamma \sum_{s_{t+1}} \mathbb{P}(s_{t+1}|s_t,a^\pi_t) V^\pi_{cpt}(s_{t+1})). \nonumber
\end{align}
Denote the maximum \emph{CPT-V }at a state $s$ by $V^*_{cpt}(s)$. 
Then, $V^*_{cpt}(s) = \sup_\pi V^\pi_{cpt}(s)$. 

\begin{remark}
To motivate the construction of this framework, let the random variable $R(s_0) = \sum_{i=0}^\infty \gamma^i r(s_i, a^\pi_i)$ denote the infinite horizon cumulative discounted reward starting from state $s_0$. 
The objective in a typical RL problem is 
to determine a policy $\pi$ to maximize the expected reward, denoted $\mathbb{E}_\pi[R(s_0)]$. 
The linearity of the expectation operator allows us to write $\mathbb{E}_\pi[R(s_0)] = \mathbb{E} [r(s_0, a_0) + \gamma \mathbb{E} [r(s_1, a_1) + \dots |s_1] |s_0]$. 
In this work, we are interested in maximizing the sum of CPT-based rewards over the horizon of interest. 
This will correspond to replacing the conditional expectation at each time-step with $\rho_{cpt}(\cdot)$.
\end{remark}

We defined a \emph{CPT-Q-iteration} operator in  \cite{ramasubramanian2021reinforcement}: 
\begin{align}
(\mathcal{T}_\pi Q^\pi_{cpt})(s,a)&:=\rho_{cpt}(r(s, a) \label{CPT-Q-Operator}\\ &+ \gamma \sum_{s'} \mathbb{P}(s'|s,a) \sum_{a'} \pi(a'|s')Q^\pi_{cpt}(s',a')), \nonumber
\end{align}
which was instrumental in establishing the convergence of CPT-Q-learning in Equation (\ref{CPT-Q-Iterative}). 
We state a result from \cite{ramasubramanian2021reinforcement}. 

\begin{definition}
A policy $\pi'$ is said to be \emph{improved} compared to policy $\pi$ if and only if for all $s \in S$, $V^{\pi'}_{cpt}(s) \geq V^\pi_{cpt}(s)$. 
\end{definition}

\begin{proposition}\cite{ramasubramanian2021reinforcement} 
Let the functions $w^+, w^-, u^+, u^-$ be according to Definition \ref{CPTValueDefn}. 
Assume that $u^+, u^-$ are invertible and differentiable with monotonically non-increasing derivatives. 
Consider policies $\pi$ and $\pi'$ such that $Q^{\pi'}_{cpt}(s,a) \geq Q^\pi_{cpt}(s,a)$ for all $(s,a) \in S \times A$, and $\pi'$ is improved compared to $\pi$. 
Then, $(\mathcal{T}_\pi Q^\pi_{cpt})$ is monotone (i.e., $(\mathcal{T}_{\pi'} Q^{\pi'}_{cpt}) \geq (\mathcal{T}_\pi Q^\pi_{cpt})$) and a contraction (i.e., $\||\mathcal{T}_\pi Q^1_{cpt} - \mathcal{T}_\pi Q^2_{cpt}|| \leq \gamma ||Q_{cpt}^1 - Q_{cpt}^2||$). 
%
\end{proposition}

The focus of this paper is to develop techniques to ensure that observation of $Q^\pi_{cpt}(s,a)$ for any $(s,a)$ does not provide any meaningful or distinguishing information between rewards $r(s,a)$ and $r'(s,a)$ as long as $||r-r'||_\infty \leq 1$.

\section{Privacy-Preserving CPT-Based RL}\label{Sec:PP-CPT-RL}

This section presents our main results that demonstrate that incorporating a privacy-preserving mechanism into a prospect-theoretic framework ensures that an agent can learn behaviors to enable it to assess and quantify risk, and ensure that its decision making is hidden from external parties. 
We first detail the design of a privacy-preserving CPT-based reinforcement learning (PP-CPT-RL) algorithm that will allow the agent to learn policies in a manner such that differential privacy of value functions is preserved. 
We then analyze the guarantees on privacy provided by PP-CPT-RL. 

\subsection{Algorithm}\label{CPTRLAlgos}
\begin{algorithm}[!h]
	\small
	\caption{CPT-PP-Estimation}
	\label{algo:CPT-PP-Estimation}
	\begin{algorithmic}[1]
		\REQUIRE{State $s$, action $a$, current policy $\pi$,  max. samples $N_{max}$}
		\STATE{\textbf{Initialize} $n = 1$; $X_{0}:=0$; $s_* \leftarrow s$}
		\REPEAT
		\STATE{Take action $a$, observe $r(s,a)$ and next state $s'$}
		\STATE{Determine $PrivPres_b(s')$ for each $b$ using Algorithm \ref{algo:PPLevel-Action}}
		\STATE{$X_n:=r(s, a) + \gamma \sum_b \pi(b|s') (Q_{cpt}^\theta(s', b)+PrivPres_b(s'))$}
		\IF{$X_n > X_{0}$}
		\STATE{$s_* \leftarrow s'$}
		\STATE{$X_0 \leftarrow X_n$}
		\ENDIF
		\STATE{$n \leftarrow n+1$}
		\UNTIL{$n>N_{max}$}
		\STATE{Arrange samples $\{X_i\}$ in ascending order: $X_{[1]} \leq X_{[2]} \leq \dots$}
		\STATE{Let:
		{\small
		\begin{align*}
		\rho_{cpt}^+:&=\sum_{i=1}^{N_{max}} u^+(X_{[i]})(w^+(\frac{N_{max}+i-1}{N_{max}}) \\
		&\qquad- w^+(\frac{N_{max}-i}{N_{max}}))\\
		\rho_{cpt}^-:&=\sum_{i=1}^{N_{max}} u^-(X_{[i]})(w^-(\frac{i}{N_{max}}) - w^-(\frac{i-1}{N_{max}}))
		\end{align*}
		}%
		}
		\STATE{$\rho_{cpt}(r(s,a)+\gamma \sum_b \pi(b|\cdot) (Q_{cpt}^\theta(\cdot, b)+PrivPres_b(\cdot))):= \rho_{cpt}^+ - \rho_{cpt}^-$}
		\RETURN{$\rho_{cpt}(\cdot), s_*$}
	\end{algorithmic}
\end{algorithm}
\begin{algorithm}[!h]
	\small
	\caption{PP-CPT-RL}
	\label{algo:CPT-PrivPres}
	\begin{algorithmic}[1]
		\REQUIRE{Learning rate $\alpha$; horizon $T_{max}$; discount $\gamma$; target privacy level $(\epsilon, \delta)$; batch size $B$}
		\STATE{\textbf{Initialize} Parameters $\theta$; value functions $Q_{cpt}^\theta(s,a)$, $T = 1$, linked list $PrivPres = []$}
		\STATE{Determine noise level $\sigma$ based on $(\epsilon, \delta)$ using Theorem \ref{ThmCPT-DiffPriv}}
		\REPEAT
		\STATE{Initialize $s \in S$, $\bar{b} = 1$, $BatchLoss_{cpt}^\theta = 0$}
		\REPEAT
		\STATE{Add $s$ to linked list $PrivPres$}
		\STATE{Determine $PrivPres_a(s)$ for each $a$ using Algorithm \ref{algo:PPLevel-Action}}
		\STATE{Choose $a$ from $\arg \max_a [Q_{cpt}^\theta(s,a)+PrivPres_a(s)]$}
		\STATE{Obtain $\rho_{cpt}(\cdot), s_*$ from Algorithm \ref{algo:CPT-PP-Estimation}; $Q_{Targ}^\theta:=\rho_{cpt}(\cdot)$}
		\STATE{$\min_{\theta} Loss_{cpt}^\theta:=0.5[Q_{Targ}^\theta - (Q_{cpt}^\theta(s,a)+PrivPres_a(s))]^2$}
		\STATE{$BatchLoss_{cpt}^\theta:=BatchLoss_{cpt}^\theta + Loss_{cpt}^\theta$}
		\STATE{$s \leftarrow s_*$}
		\STATE{$\bar{b} \leftarrow \bar{b}+1$}
		\UNTIL{$\bar{b}>B$}
		\STATE{Update parameters $\theta \leftarrow \theta -  \alpha \frac{1}{B}\nabla_\theta BatchLoss_{cpt}^\theta$}
		\STATE{$T \leftarrow T+1$}
		\UNTIL{$T>T_{max}$}
		\RETURN{$Q_{cpt}^\theta(s,a)$}
	\end{algorithmic}
\end{algorithm}
\begin{algorithm}[!h]
	\small
	\caption{PPLevel-Action}
	\label{algo:PPLevel-Action}
	\begin{algorithmic}[1]
		\REQUIRE{State $s$, noise level $\sigma$, list $PrivPres$}
		\FOR{each action $a$}
		\STATE{$\mu_a =K_{10}K_{00}^{-1}$}
		\STATE{$d_a = K_{11}-K_{10}K_{00}^{-1}K_{10}^T$}
		\STATE{$PrivPres_a(s) \sim \mathcal{N}(\mu_a, \sigma d_a)$}
		\ENDFOR
		\RETURN{$PrivPres(s):=[PrivPres_a(s)]_{a=1}^m$}
	\end{algorithmic}
\end{algorithm}

We present an algorithm based on temporal difference (TD) techniques for PP-CPT-RL. 
TD techniques seek to learn value functions using episodes of experience.
An experience episode comprises a sequence of states, actions, and rewards when following a policy $\pi$. 
The predicted values at any time-step is updated in a way to bring it closer to the prediction of the same quantity at the next time-step. 
In order to support environments with large, possibly continuous state spaces, our approach is based on \emph{deep Q-learning} \cite{mnih2015human}, where Q-values are parameterized by a deep neural network\footnote{This parameterization is analogous to the Q-table that is typically seen in Q-learning \cite{sutton2018reinforcement}. Updating entries of the Q-table is then equivalent to updating the values of the parameter $\theta$ using a gradient-based method.}. 

From Equations (\ref{CPTValue}) and (\ref{CPTValueDefn}), we observe that  $\rho_{cpt}$ is defined in terms of a weighting function applied to a cumulative probability distribution. 
We first use a technique proposed in \cite{jie2018stochastic} to estimate the CPT-value $\rho_{cpt}$ to use TD-methods.

\subsubsection{Calculating $\rho_{cpt}$ from samples}

Algorithm \ref{algo:CPT-PP-Estimation} is a procedure to obtain multiple samples of the random variable $r(s,a)+\gamma V_{cpt}(s')$. 
These samples are then used to estimate $\rho_{cpt}(r(s,a)+\gamma V_{cpt}(s'))$ (since $V_{cpt}(s) = \sum_a \pi(\cdot|s)Q_{cpt}(s)$). 
This way to estimate the CPT-value of a random variable was proposed in \cite{jie2018stochastic}, and was shown to be asymptotically consistent. 
In order to obtain these estimates in a privacy-preserving manner, we use Algorithm \ref{algo:PPLevel-Action} to determine the quantum of noise that needs to be added for each action at any state. 

\subsubsection{Privacy-preserving CPT-Q-Learning}

Algorithm \ref{algo:CPT-PrivPres} is a technique to learn policies for an RL-agent with a CPT-based objective in a privacy-preserving manner. 
The target privacy level $(\epsilon, \delta)$ determines the noise level that will be needed to ensure differential privacy (\emph{Line 2}). 
The linked-list $PrivPres$ maintains a record of states visited along sample trajectories in the replay buffer, since parameters of the Gaussian noise added to the value functions at each step will depend on the `chain' of states in this list (\emph{Line 7}). 
The training process consists of working with batches of sample trajectories that have been collected in a \emph{replay buffer}. 
A quadratic loss function that measures the TD-error with respect to a \emph{target $Q-$network} (whose parameters are kept fixed) is minimized in \emph{Line 10}. 
The parameters of the target network are updated after examining all the samples from a batch using a gradient-based method (\emph{Line 15}) \cite{mnih2015human}. 
\emph{Lines 10 and 15} together comprise the TD-update.

\subsubsection{Determining noise level for each action}

Algorithm \ref{algo:PPLevel-Action} is used to determine the quantum of noise that needs to be added for each action at state $s$. 
The noise is added using a Gaussian process mechanism. 
Following the notation established in Sec. \ref{GausProc}, and $I$ denoting the identity matrix, we define $\mu_a$ and $d_a$ (\emph{Lines 2-3}) as: 
\begin{align}
&K_{10}K_{00}^{-1}=\frac{exp(-\beta_n)}{1+exp(-2\beta_n)}\begin{bmatrix}
1&1&0&\dots&0&0\\0&1&1&\dots&0&0\\\vdots&\vdots&\vdots&\ddots&\vdots&\vdots\\0&0&0&\dots&1&0\\0&0&0&\dots&1&1
\end{bmatrix}\\
&K_{11}-K_{10}K_{00}^{-1}K_{10}^T=\frac{1-exp(-2\beta_n)}{1+exp(-2\beta_n)}I
\end{align}

\subsection{Privacy Analysis} 
Our objective is to ensure that observing or querying the value function in Algorithm \ref{algo:CPT-PrivPres} will not reveal useful information about rewards received by an agent. 
This is formally stated in Theorem \ref{ThmCPT-DiffPriv}, which establishes a guarantee on the privacy of Algorithm \ref{algo:CPT-PrivPres}.  
We make an assumption to ensure that the neural networks used to learn 
$Q^\theta_{cpt}(s,a)$ define a complete vector space of functions equipped with a norm that is a combination of $L^p-$norms of the functions along with some of its derivatives. 
Then, we state intermediate results which establish i) conditions for $(\epsilon, \delta)-$differential privacy of a mechanism $\mathbb{M}$ \cite{dwork2014algorithmic, hall2013differential, wang2019privacy}, 
ii) that sample paths of a Gaussian process generated on a reproducing kernel Hilbert space (RKHS) are bounded with high probability \cite{wang2019privacy},  
and iii) an expression for the RKHS norm \cite{wang2019privacy, lalley2013gaussian}.

\begin{assumption}
The neural network used to approximate $Q^\pi_{cpt}(s,a)$ in Algorithm \ref{algo:CPT-PrivPres} has a finite number of parameters, a finite number of layers, and the gradients of its activation functions are bounded. 
\end{assumption}
%
%
%

\begin{proposition}\label{PropDiffPriv}
Let $\Delta_{\mathbb{M}}$ be the sensitivity of a mechanism $\mathbb{M}$. The following statements hold:  
\begin{enumerate}
\item 
If $\epsilon \in (0,1)$ and $\sigma \geq \sqrt{2 \ln (1.25/\delta)}\Delta_{\mathbb{M}}/\epsilon$, then $\mathbb{M}(d)+y$ is $(\epsilon, \delta)-$differentially private, where $y \sim \mathcal{N}(0, \sigma^2 I)$. 
\item 
If $\epsilon \in (0,1)$ and $\sigma \geq \sqrt{2 \ln (1.25/\delta)}\Delta_{\mathbb{M}}/\epsilon$, then $\mathbb{M}(d)+g$ is $(\epsilon, \delta)-$differentially private, where $g \sim \mathcal{G}(0, \sigma^2 K)$, and $\mathcal{U}$ is an RKHS with kernel function $K$.
\end{enumerate}
\end{proposition}

\begin{lemma}
\label{LemmaGPTail}
Consider a sample path $f$ generated by a GP $\mathcal{G}(0, \sigma^2K)$ on an RKHS with kernel $K(x,y):=exp(-\beta||x-y||_1)$. Then, $f^*:=\max f(x)$ exists almost surely, and for any $u > 0$, $\mathbb{P}(f^* \geq 8.68 \sqrt{\beta}\sigma + u) \leq exp(-u^2/2)$. 
\end{lemma}

\begin{lemma}
\label{LemmaRKHSNorm}
The RKHS associated with kernel $K(x,y):=exp(-\beta||x-y||_1)$ consists of all continuous functions $\phi$ with finite $L^2-$norm whose derivatives $\phi'$ also have finite $L^2-$norm. 
Further, the RKHS norm $||\phi||^2_{\mathcal{H}} \leq (1+\beta/2)(\phi(x))^2+L^2/2\beta$, where $L$ is the Lipschitz constant.
\end{lemma}

We are now ready to establish our main result that establishes a privacy guarantee on value functions learned using Algorithm \ref{algo:CPT-PrivPres} when rewards are sufficiently close to each other using differential privacy. 

\begin{theorem}\label{ThmCPT-DiffPriv}
Let the functions $u^+, u^-$ in Definition \ref{CPTValueDefn} be invertible and differentiable with monotonically non-increasing derivatives. 
Let $2k > 8.68\sqrt{\beta}\sigma$, $L$ be the Lipschitz constant of the value function approximation yielded by the neural network parameterized by $\theta$, $B$ be the batch size, 
and $\alpha$ be the learning rate of Algorithm \ref{algo:CPT-PrivPres}. 
Assume that $\sigma \geq \sqrt{2 \ln (1.25/\delta)}(L^2(1/\beta^2 + 1/\beta))/\epsilon$ and $\beta = \Big[\frac{\alpha}{B}(4k+2c_{\max}+1)\Big]^{-1}$, where $c_{\max} \geq (r+\gamma Q_{cpt} (s_*)-Q_{cpt}(s))$. 

Then, each iteration of Algorithm \ref{algo:CPT-PrivPres} is $(\epsilon, \delta +  exp(-(2k-8.68\sqrt{\beta}\sigma)^2/2)-$ differentially private with respect to neighboring reward functions $r, r'$ such that $||r-r'||_\infty \leq 1$. 
\end{theorem}

\begin{proof} 
Let $Q_{cpt}$ and $Q'_{cpt}$ be the state-action value functions learned using Algorithm \ref{algo:CPT-PrivPres} corresponding to rewards $r$ and $r'$ respectively. 
Consider the update-step in \emph{Line 15} of Algorithm \ref{algo:CPT-PrivPres}. 
If $Q^{old}_{cpt}$ denotes the (fixed) value function before the update and $L$ is the Lipschitz constant of the value function approximation, we have: 
\begin{align*}
||Q_{cpt} - Q^{old}_{cpt}||&\leq \frac{\alpha L}{B}||r+\gamma Q^{old}_{cpt} (s_*)-Q^{old}_{cpt}(s) \\&\qquad\quad+ PrivPres_a(s_*)-PrivPres_a(s)||, \\
||Q'_{cpt} - Q^{old}_{cpt}||&\leq \frac{\alpha L}{B}||r'+\gamma Q^{old}_{cpt} (s_*)-Q^{old}_{cpt}(s) \\&\qquad\quad+ PrivPres_a(s_*)-PrivPres_a(s)||.
\end{align*}
Set $u = 2k-8.68\sqrt{\beta}\sigma$ in Lemma \ref{LemmaGPTail}. 
Then $||PrivPres_a(s_*) - PrivPres_a(s)|| \leq 2k$ with probability $1-exp(-(2k-8.68\sqrt{\beta}\sigma)^2/2)$. 
When $||r-r'||_{\infty} \leq 1$, and $(r+\gamma Q^{old}_{cpt} (s_*)-Q^{old}_{cpt}(s)) \leq c_{\max}$, we have: 
\begin{align*}
||Q_{cpt} - Q'_{cpt}|| &\leq ||Q_{cpt} - Q^{old}_{cpt}|| + ||Q'_{cpt} - Q^{old}_{cpt}||\\
&\leq \frac{\alpha L}{B}(4k+2c_{\max}+1)
\end{align*}
with probability $1-exp(-(2k-8.68\sqrt{\beta}\sigma)^2/2)$. 

Defining $\phi: = Q_{cpt} - Q'_{cpt}$, from Lemma \ref{LemmaRKHSNorm}, we have:
\begin{align*}
||\phi||^2_{\mathcal{H}} &\leq (1+\frac{\beta}{2})\Big[\frac{\alpha L}{B}(4k+2c_{\max}+1)\Big]^2+\frac{L^2}{2\beta}
\end{align*}
Choose $\beta = \Big[\frac{\alpha}{B}(4k+2c_{\max}+1)\Big]^{-1}$. 
Then $||\phi||^2_{\mathcal{H}}  \leq L^2(1/\beta^2 + 1/\beta)$. 

Using Proposition \ref{PropDiffPriv}, if $\sigma \geq \sqrt{2 \ln (1.25/\delta)}||\phi||_{\mathcal{H}}/\epsilon$, then adding $PrivPres(s) \sim \mathcal{G}(0, \sigma^2K)$ to the value function $Q_{cpt}$ ensures that each iteration of Algorithm \ref{algo:CPT-PrivPres} is $(\epsilon,\delta +  exp(-(2k-8.68\sqrt{\beta}\sigma)^2/2)-$differentially private whenever $||r-r'||_\infty \leq 1$.  
\end{proof}

\section{Experimental Evaluation}\label{Sec:Expt}

This section presents an evaluation of the PP-CPT-RL algorithm. 
We compare behaviors learned when the agent adds different amounts of a calibrated noise to its value function in order to maintain privacy of its decision making with a baseline when this noise is not added. 
\begin{table*}[]
    \centering
    \begin{tabular}{|c|c|c|c|c|c|c|c|c|c|}
\hline
 &  & \multicolumn{4}{c|}{CPT} &   \multicolumn{4}{c|}{No CPT}\\  \cline{3-10}
   &  method  & obs1  & obs2  & obs3 & obs4  & obs1 & obs2& obs3 & obs4   \\  \hline \hline
 \multirow{3}{*}{\rotatebox[origin=c]{90}{Max Q}} &   DP-5 & 3.74& 0.74 &0.02&  0 & 7.995 & 0.63 & 0.06 & 0.08 \\    \cline{2-10}
    & DP-1&1.31& 2.63 &0.85&0.2 &4.14  & 1.49 & 0.165 & 2.375 \\      \cline{2-10}
    & No DP &1.125&  4.26 & 0.055& 0.81 &4.57& 2.59&  0.055 &  0.05  \\    \hline\hline
  
    \multirow{3}{*}{\rotatebox[origin=c]{90}{Rand Q}}   &DP-5  &7.58  &0.115& 0.365& 0.165& 8.595& 0.705 &0.28 & 0.215   \\     \cline{2-10}
     &DP-1 &1.595& 1.825& 0.575& 0.815& 3.615 & 1.56 & 0.525 &1.045  \\     \cline{2-10}
    & No DP& 0.54 &3.49 & 1.695 & 3.605& 2.415  &4.43 & 1.47 &3.385 \\      \hline
       
       \hline
    \end{tabular}
    \caption{Average number of visits to an obstacle for different privacy levels. 
    We compare cases when the agent chooses an action according to the highest Q-value ($Max. Q$) and when an action is chosen at random ($Rand Q$). 
    The agent visits an obstacle fewer times when maximizing a CPT-based reward ($CPT$) than an expected reward ($No$ $CPT$). 
    When penalties for $\{obs1, obs2, obs3, obs4\}$ are $\{50, 25, 10, 5\}$, that CPT-based objective is more effective in avoiding visits to obstacles with higher penalties. 
    The number of visits to an obstacle is also higher for a higher privacy level.}
    \label{tab:CPTresults}
\end{table*}
%

\begin{figure}
    \centering
    \includegraphics[scale=0.30]{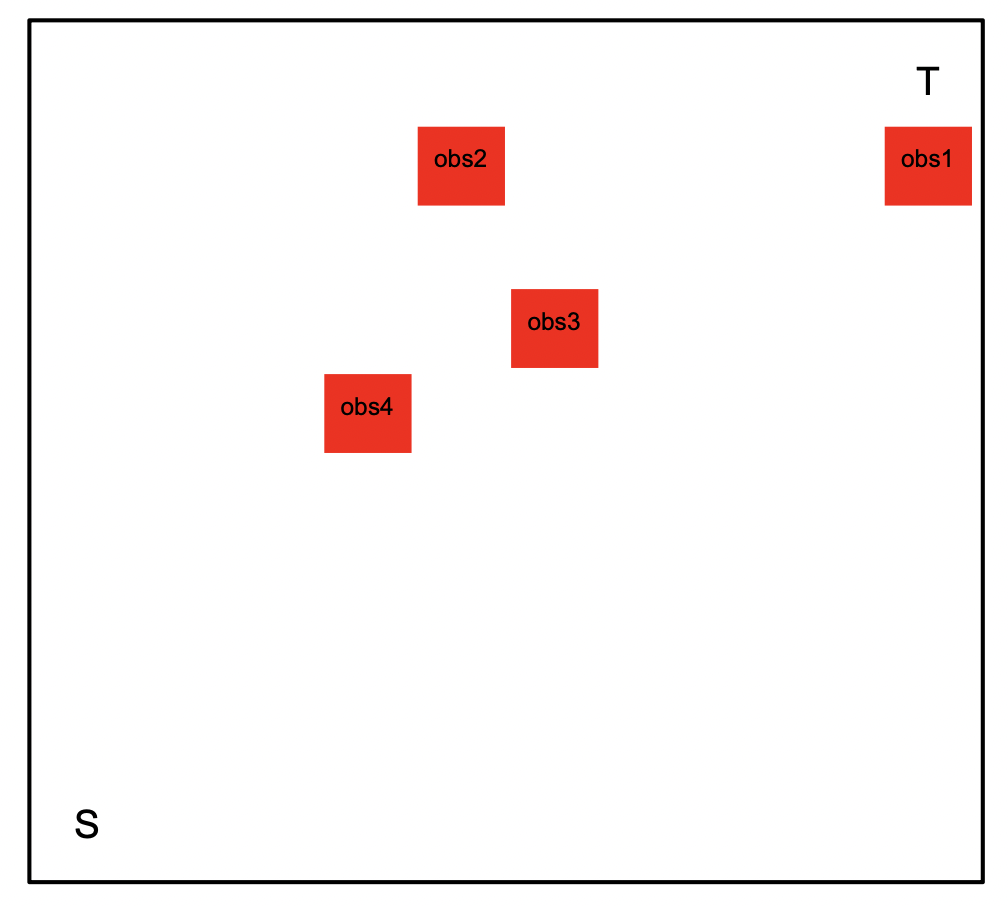}
    \caption{The $10 \times 10$ square region on which the PP-CPT-RL algorithm is evaluated. The agent needs to learn a policy to reach target $T$ from start $S$. There are obstacles in the environment (red squares), and the agent incurs a different cost when encountering each obstacle.}
    \label{fig:grid}
\end{figure}

We assume that the agent starts from the state `S' at the bottom left corner, and the target state `T' is at the top right corner of a $10 \times 10$ square region shown in Fig. \ref{fig:grid}. 
At each state, the agent can take one of four possible actions, $\{left,right,up,down\}$. 
If an action is allowed at a state then the transition to the intended next state happens with probability $0.9$ and with a probability of $0.1$ to another state. 
Suppose the current state of the agent is the position $(x,y)$, and the agent takes the action $right$. 
Then intended next state is determined as $(x', y')$ such that $Int(x') = Int(x)+1$ and $Int(y') = Int(y)$, where $Int(\cdot)$ denotes the integer part of the argument. 
The intended next state for other actions is similarly determined. 
If an action is not available in a state (e.g., \emph{down} at the Start) the agent remains in that state. 
The agent will have to avoid obstacles in order to reach the target. 
The discount factor $\gamma$ is set to $0.9$, and the utility and weighting functions are chosen as: 
\begin{align*}
&u^+(x) = |x|^{0.88}; \quad \omega^+(\kappa) = \frac{\kappa^{0.61}}{(\kappa^{0.61}+(1-\kappa)^{0.61})^{\frac{1}{0.61}}};\\
&u^-(x)=|x|^{0.88};\quad \omega^-(\kappa) = \frac{\kappa^{0.69}}{(\kappa^{0.69}+(1-\kappa)^{0.69})^{\frac{1}{0.69}}}.
\end{align*}
%
%
\begin{figure}
    \centering
    \includegraphics[scale=0.5]{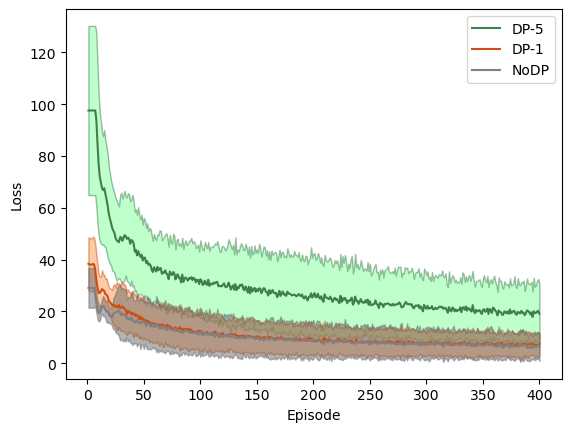}
    \caption{Loss values (\emph{Lines 10-11} of Algorithm \ref{algo:CPT-PrivPres}), averaged over 20 runs. Shaded regions indicate variance of the loss. When $\sigma = 1$ (red curve, denoted $DP-1$), the average loss and its variance is comparable to the setting without differential privacy (blue curve, denoted $NoDP$). Choosing $\sigma = 5$ (green curve, denoted $DP-5$) provides improved guarantees on privacy according to Theorem \ref{ThmCPT-DiffPriv}, but results in a higher loss, thus demonstrating a \emph{privacy-utility tradeoff}.} 
        \label{fig:loss}
\end{figure}

Figure \ref{fig:loss} compares the performance of Algorithm \ref{algo:CPT-PrivPres} for different privacy levels. 
We observe that choosing $\sigma = 1$ in \emph{Line 2} results in a performance that is as good as without differential privacy (i.e., no noise added to the value function), since the loss values and is variance is almost identical in both cases. 
A higher value of $\sigma = 5$, while providing improved guarantees on privacy, also results in a higher magnitude of loss. 
This demonstrates that there is an inherent trade-off between privacy and utility. 

In order to reason about agent behavior when learning policies using Algorithm \ref{algo:CPT-PrivPres}, we assume that each obstacle has a different penalty (to represent the relative severity of encountering the obstacle). 
Table \ref{tab:CPTresults} compares the number of visits to the obstacle regions. 
We compare cases when the agent chooses an action according to the highest Q-value ($Max. Q$) and when an action is chosen proportional to its probability ($Rand Q$). 
We make the following observations: 
\begin{itemize}
\item the agent visits an obstacle fewer times when maximizing a CPT-based objective than when maximizing an expected utility ($No$ $CPT$); 
\item the CPT-based objective is more effective in avoiding visits to obstacles with highest penalties; 
\item the number of visits to an obstacle is higher for a higher privacy level- this is consistent with the higher loss values seen in Fig. \ref{fig:loss}. 
\end{itemize}
These results are consistent with intuition, and will be aligned with that of a human user placed in the same environment. 
Developing a mathematical characterization of these properties is a promising direction of future research. 

\section{Conclusion}\label{Sec:Conclusion}

This paper presented a way to enable a reinforcement learning (RL) agent to learn behaviors that are consistent with human traits of assessing risk and a desire to keep decision making private. 
We used cumulative prospect theory (CPT) to quantify risk-sensitive behavior, and differential privacy to characterize privacy-preserving decision making. 
We designed an algorithm, PP-CPT-RL, to enable the agent to learn policies to maximize a CPT-based objective. 
Through adding a calibrated noise to CPT-based value functions we established guarantees on privacy when rewards were sufficiently close to each other. 
Experimental evaluation of PP-CPT-RL showed that agents can learn optimal policies in a privacy-preserving manner. 
Our experiments also revealed a privacy-utility tradeoff, and showed that agent behavior is consistent with a human placed in the same environment. 

Future work will seek to develop mathematically rigorous characterizations of the privacy-utility trade-off, establish privacy guarantees for other classes of RL algorithms (e.g., actor-critic), and examine extensions to the multi-agent case. 
%
\bibliographystyle{IEEEtran}
\bibliography{CDC22-RLPriv-References}

\begin{thebibliography}{10}
\providecommand{\url}[1]{#1}
\csname url@samestyle\endcsname
\providecommand{\newblock}{\relax}
\providecommand{\bibinfo}[2]{#2}
\providecommand{\BIBentrySTDinterwordspacing}{\spaceskip=0pt\relax}
\providecommand{\BIBentryALTinterwordstretchfactor}{4}
\providecommand{\BIBentryALTinterwordspacing}{\spaceskip=\fontdimen2\font plus
\BIBentryALTinterwordstretchfactor\fontdimen3\font minus
  \fontdimen4\font\relax}
\providecommand{\BIBforeignlanguage}[2]{{%
\expandafter\ifx\csname l@#1\endcsname\relax
\typeout{** WARNING: IEEEtran.bst: No hyphenation pattern has been}%
\typeout{** loaded for the language `#1'. Using the pattern for}%
\typeout{** the default language instead.}%
\else
\language=\csname l@#1\endcsname
\fi
#2}}
\providecommand{\BIBdecl}{\relax}
\BIBdecl

\bibitem{sutton2018reinforcement}
R.~S. Sutton and A.~G. Barto, \emph{Reinforcement Learning: {A}n
  Introduction}.\hskip 1em plus 0.5em minus 0.4em\relax MIT Press, 2018.

\bibitem{bertsekas2017dynamic}
D.~P. Bertsekas, \emph{Dynamic {P}rogramming and {O}ptimal {C}ontrol, {V}ol. 1,
  4th Ed.}\hskip 1em plus 0.5em minus 0.4em\relax Athena Scientific, 2017.

\bibitem{puterman2014markov}
M.~L. Puterman, \emph{Markov decision processes: {D}iscrete stochastic dynamic
  programming}.\hskip 1em plus 0.5em minus 0.4em\relax John Wiley \& Sons,
  2014.

\bibitem{hafner2011reinforcement}
R.~Hafner and M.~Riedmiller, ``Reinforcement learning in feedback control,''
  \emph{Machine Learning}, vol.~84, pp. 137--169, 2011.

\bibitem{mnih2015human}
V.~Mnih \emph{et~al.}, ``Human-level control through deep reinforcement
  learning,'' \emph{Nature}, vol. 518, no. 7540, 2015.

\bibitem{silver2016mastering}
D.~Silver \emph{et~al.}, ``Mastering the game of {G}o with deep neural networks
  and tree search,'' \emph{Nature}, vol. 529, no. 7587, 2016.

\bibitem{zhang2019deep}
C.~Zhang, P.~Patras, and H.~Haddadi, ``Deep learning in mobile and wireless
  networking: {A} survey,'' \emph{IEEE Communications Surveys \& Tutorials},
  vol.~21, no.~3, pp. 2224--2287, 2019.

\bibitem{sadigh2016planning}
D.~Sadigh, S.~Sastry, S.~A. Seshia, and A.~D. Dragan, ``Planning for autonomous
  cars that leverage effects on human actions.'' in \emph{Robotics: Science and
  Systems}, 2016.

\bibitem{yan2018data}
Z.~Yan and Y.~Xu, ``Data-driven load frequency control for stochastic power
  systems: {A} deep reinforcement learning method with continuous action
  search,'' \emph{IEEE Transactions on Power Systems}, vol.~34, no.~2, 2018.

\bibitem{you2019advanced}
C.~You, J.~Lu, D.~Filev, and P.~Tsiotras, ``Advanced planning for autonomous
  vehicles using reinforcement learning and deep inverse {RL},'' \emph{Robotics
  and Autonomous Systems}, vol. 114, pp. 1--18, 2019.

\bibitem{seshia2015formal}
S.~A. Seshia, D.~Sadigh, and S.~S. Sastry, ``Formal methods for semi-autonomous
  driving,'' in \emph{{ACM/EDAC/IEEE} {D}esign {A}utomation
  {C}onference}.\hskip 1em plus 0.5em minus 0.4em\relax IEEE, 2015, pp. 1--5.

\bibitem{nikolakis2019cyber}
N.~Nikolakis, V.~Maratos, and S.~Makris, ``A cyber physical system approach for
  safe human-robot collaboration in a shared workplace,'' \emph{Robotics and
  Computer-Integrated Manufacturing}, vol.~56, pp. 233--243, 2019.

\bibitem{xiao2020fresh}
B.~Xiao, Q.~Lu, B.~Ramasubramanian, A.~Clark, L.~Bushnell, and R.~Poovendran,
  ``{FRESH}: {I}nteractive reward shaping in high-dimensional state spaces
  using human feedback,'' in \emph{International Conference on Autonomous
  Agents and MultiAgent Systems}, 2020, pp. 1512--1520.

\bibitem{kahneman1979prospect}
D.~Kahneman and A.~Tversky, ``Prospect theory: {A}n analysis of decision under
  risk,'' \emph{Econometrica}, vol.~47, no.~2, pp. 263--292, 1979.

\bibitem{schmidt2003reference}
U.~Schmidt, ``Reference dependence in cumulative prospect theory,''
  \emph{Journal of Mathematical Psychology}, vol.~47, no.~2, pp. 122--131,
  2003.

\bibitem{barberis2013thirty}
N.~C. Barberis, ``Thirty years of prospect theory in economics: {A} review and
  assessment,'' \emph{Journal of Economic Perspectives}, vol.~27, no.~1, pp.
  173--96, 2013.

\bibitem{tversky1992advances}
A.~Tversky and D.~Kahneman, ``Advances in prospect theory: {C}umulative
  representation of uncertainty,'' \emph{Journal of Risk and uncertainty},
  vol.~5, no.~4, pp. 297--323, 1992.

\bibitem{ramasubramanian2021reinforcement}
B.~Ramasubramanian, L.~Niu, A.~Clark, and R.~Poovendran, ``Reinforcement
  learning beyond expectation,'' in \emph{Conference on Decision and Control
  (CDC)}.\hskip 1em plus 0.5em minus 0.4em\relax IEEE, 2021.

\bibitem{dwork2014algorithmic}
C.~Dwork, A.~Roth \emph{et~al.}, ``The algorithmic foundations of differential
  privacy,'' \emph{Foundations and Trends{\textregistered} in Theoretical
  Computer Science}, vol.~9, no. 3--4, pp. 211--407, 2014.

\bibitem{liao2019prospect}
G.~Liao, X.~Chen, and J.~Huang, ``Prospect theoretic analysis of
  privacy-preserving mechanism,'' \emph{IEEE/ACM Transactions on Networking},
  vol.~28, no.~1, pp. 71--83, 2019.

\bibitem{wang2019privacy}
B.~Wang and N.~Hegde, ``Privacy-preserving {Q}-learning with functional noise
  in continuous spaces,'' \emph{Advances in Neural Information Processing
  Systems}, vol.~32, 2019.

\bibitem{ng2000algorithms}
A.~Y. Ng and S.~J. Russell, ``Algorithms for inverse reinforcement learning,''
  in \emph{International Coference on Machine Learning}, 2000, pp. 663--670.

\bibitem{abbeel2004apprenticeship}
P.~Abbeel and A.~Y. Ng, ``Apprenticeship learning via inverse reinforcement
  learning,'' in \emph{International Coference on Machine Learning}, 2004.

\bibitem{shen2013risk}
Y.~Shen, W.~Stannat, and K.~Obermayer, ``Risk-sensitive {M}arkov control
  processes,'' \emph{SIAM Journal on Control and Optimization}, vol.~51, no.~5,
  pp. 3652--3672, 2013.

\bibitem{shen2014risk}
Y.~Shen, M.~J. Tobia, T.~Sommer, and K.~Obermayer, ``Risk-sensitive
  reinforcement learning,'' \emph{Neural computation}, vol.~26, no.~7, pp.
  1298--1328, 2014.

\bibitem{markowitz1952portfolio}
H.~Markowitz, ``Portfolio selection,'' \emph{The Journal of Finance}, vol.~7,
  no.~1, pp. 77--91, 1952.

\bibitem{tamar2012policy}
A.~Tamar, D.~Di~Castro, and S.~Mannor, ``Policy gradients with variance related
  risk criteria,'' in \emph{International Coference on Machine Learning}, 2012,
  pp. 1651--1658.

\bibitem{mannor2013algorithmic}
S.~Mannor and J.~N. Tsitsiklis, ``Algorithmic aspects of mean--variance
  optimization in {M}arkov decision processes,'' \emph{European Journal of
  Operational Research}, vol. 231, no.~3, pp. 645--653, 2013.

\bibitem{howard1972risk}
R.~A. Howard and J.~E. Matheson, ``Risk-sensitive {M}arkov decision
  processes,'' \emph{Management Science}, vol.~18, no.~7, pp. 356--369, 1972.

\bibitem{whittle1990risk}
P.~Whittle, \emph{Risk-sensitive optimal control}.\hskip 1em plus 0.5em minus
  0.4em\relax Wiley, 1990.

\bibitem{borkar2002q}
V.~S. Borkar, ``Q-learning for risk-sensitive control,'' \emph{Mathematics of
  Operations Research}, vol.~27, no.~2, pp. 294--311, 2002.

\bibitem{rockafellar2002conditional}
R.~T. Rockafellar and S.~Uryasev, ``Conditional value-at-risk for general loss
  distributions,'' \emph{Journal of banking \& finance}, vol.~26, no.~7, pp.
  1443--1471, 2002.

\bibitem{ahmadi2021constrained}
M.~Ahmadi, U.~Rosolia, M.~D. Ingham, R.~M. Murray, and A.~D. Ames,
  ``Constrained risk-averse {M}arkov decision processes,'' in \emph{AAAI
  Conference on Artificial Intelligence}, 2021.

\bibitem{chapman2021risk}
M.~P. Chapman, R.~Bonalli, K.~M. Smith, I.~Yang, M.~Pavone, and C.~J. Tomlin,
  ``Risk-sensitive safety analysis using conditional value-at-risk,''
  \emph{IEEE Transactions on Automatic Control}, 2021.

\bibitem{lindemann2020control}
L.~Lindemann, G.~J. Pappas, and D.~V. Dimarogonas, ``Control barrier functions
  for nonholonomic systems under risk signal temporal logic specifications,''
  in \emph{IEEE Conference on Decision and Control (CDC)}.\hskip 1em plus 0.5em
  minus 0.4em\relax IEEE, 2020, pp. 1422--1428.

\bibitem{majumdar2020should}
A.~Majumdar and M.~Pavone, ``How should a robot assess risk? {T}owards an
  axiomatic theory of risk in robotics,'' in \emph{Robotics Research}.\hskip
  1em plus 0.5em minus 0.4em\relax Springer, 2020, pp. 75--84.

\bibitem{chow2017risk}
Y.~Chow, M.~Ghavamzadeh, L.~Janson, and M.~Pavone, ``Risk-constrained
  reinforcement learning with percentile risk criteria,'' \emph{The Journal of
  Machine Learning Research}, vol.~18, pp. 6070--6120, 2017.

\bibitem{prashanth2018risk}
L.~A. Prashanth and M.~Fu, ``Risk-sensitive reinforcement learning: {A}
  constrained optimization viewpoint,'' \emph{arXiv:1810.09126}, 2018.

\bibitem{jie2018stochastic}
C.~Jie, L.~A. Prashanth, M.~Fu, S.~Marcus, and C.~Szepesv{\'a}ri, ``Stochastic
  optimization in a cumulative prospect theory framework,'' \emph{IEEE
  Transactions on Automatic Control}, vol.~63, no.~9, 2018.

\bibitem{lin2018probabilistically}
K.~Lin, C.~Jie, and S.~I. Marcus, ``Probabilistically distorted risk-sensitive
  infinite-horizon dynamic programming,'' \emph{Automatica}, vol.~97, pp. 1--6,
  2018.

\bibitem{borkar2021prospect}
V.~S. Borkar and S.~Chandak, ``Prospect-theoretic {Q}-learning,'' \emph{Systems
  \& Control Letters}, vol. 156, no.~10, p. 105009, 2021.

\bibitem{cortes2016differential}
J.~Cort{\'e}s, G.~E. Dullerud, S.~Han, J.~Le~Ny, S.~Mitra, and G.~J. Pappas,
  ``Differential privacy in control and network systems,'' in \emph{IEEE
  Conference on Decision and Control}, 2016, pp. 4252--4272.

\bibitem{yazdani2018differentially}
K.~Yazdani, A.~Jones, K.~Leahy, and M.~Hale, ``Differentially private {LQ}
  control,'' \emph{IEEE Transactions on Automatic Control}, 2022.

\bibitem{han2018privacy}
S.~Han and G.~J. Pappas, ``Privacy in control and dynamical systems,''
  \emph{Annual Review of Control, Robotics, and Autonomous Systems}, vol.~1,
  pp. 309--332, 2018.

\bibitem{chistikov2019asymmetric}
D.~Chistikov, A.~S. Murawski, and D.~Purser, ``Asymmetric distances for
  approximate differential privacy,'' in \emph{International Conference on
  Concurrency Theory}, 2019.

\bibitem{ramasubramanian2020privacy}
B.~Ramasubramanian, L.~Niu, A.~Clark, L.~Bushnell, and R.~Poovendran,
  ``Privacy-preserving resilience of cyber-physical systems to adversaries,''
  in \emph{IEEE Conference on Decision and Control}, 2020, pp. 3785--3792.

\bibitem{balle2016differentially}
B.~Balle, M.~Gomrokchi, and D.~Precup, ``Differentially private policy
  evaluation,'' in \emph{International Conference on Machine Learning}, 2016.

\bibitem{vietri2020private}
G.~Vietri, B.~Balle, A.~Krishnamurthy, and S.~Wu, ``Private reinforcement
  learning with {PAC} and regret guarantees,'' in \emph{International
  Conference on Machine Learning}, 2020, pp. 9754--9764.

\bibitem{garcelon2021local}
E.~Garcelon, V.~Perchet, C.~Pike-Burke, and M.~Pirotta, ``Local differential
  privacy for regret minimization in reinforcement learning,'' \emph{Advances
  in Neural Information Processing Systems}, vol.~34, 2021.

\bibitem{gohari2020privacy}
P.~Gohari, M.~Hale, and U.~Topcu, ``Privacy-preserving policy synthesis in
  markov decision processes,'' in \emph{IEEE Conference on Decision and
  Control}, 2020, pp. 6266--6271.

\bibitem{prelec1998probability}
D.~Prelec, ``The probability weighting function,'' \emph{Econometrica}, pp.
  497--527, 1998.

\bibitem{williams2006gaussian}
C.~K. Williams and C.~E. Rasmussen, \emph{Gaussian processes for machine
  learning}.\hskip 1em plus 0.5em minus 0.4em\relax MIT Press, 2006.

\bibitem{hall2013differential}
R.~Hall, A.~Rinaldo, and L.~Wasserman, ``Differential privacy for functions and
  functional data,'' \emph{The Journal of Machine Learning Research}, vol.~14,
  no.~1, pp. 703--727, 2013.

\bibitem{lalley2013gaussian}
S.~Lalley, ``Introduction to {G}aussian processes,''
  \emph{\url{https://galton.uchicago.edu/~lalley/Courses/386/GaussianProcesses.pdf}},
  2013.

\end{thebibliography}

\end{document}